\documentclass[10pt,twocolumn,letterpaper]{article}
\usepackage{cvpr}
\usepackage{times}
\usepackage{epsfig}
\usepackage{graphicx}
\usepackage{amsmath}
\usepackage{subcaption}
\usepackage{amssymb}
\usepackage{multirow}
\usepackage{multicol}
\usepackage{floatrow}
\usepackage{booktabs}
\usepackage{algorithmic}
\usepackage{float}
\usepackage{bbding}

\usepackage{array}

\newcolumntype{I}{!{\vrule width 3pt}}
\newlength\savedwidth

\newlength\savewidth

\usepackage{lipsum}

\newcommand\blfootnote[1]{%
\begingroup
\renewcommand\thefootnote{}\footnote{#1}%
\addtocounter{footnote}{-1}%
\endgroup
}

\usepackage[breaklinks=true,bookmarks=false]{hyperref}
\hypersetup{
colorlinks=true,
linkcolor=black
}
\usepackage{color}
\usepackage{enumitem}
\usepackage{arydshln}
\usepackage{bm}
\usepackage{authblk}
\usepackage{amsthm}
\usepackage[margin=4pt,font=footnotesize,labelfont=bf,labelsep=endash,tableposition=top]{caption}

\renewcommand{\texttt}[1]{ $ {{\tt #1} } $}

\usepackage{booktabs}

\usepackage{color}
\usepackage{amsthm}
\usepackage{algorithm}
\usepackage{algorithmic}
\usepackage{float}
\usepackage{bbding}
\usepackage{multirow}
\definecolor{orange}{RGB}{255,127,0}
\usepackage{amsmath,amsfonts,bm}

\def\eqref#1{equation~\ref{#1}}
\def\1{\bm{1}}

\def\vb{{\bm{b}}}
\def\vc{{\bm{c}}}

\def\vo{{\bm{o}}}
\def\vp{{\bm{p}}}
\def\vq{{\bm{q}}}

\def\vx{{\bm{x}}}

\def\vz{{\bm{z}}}

\def\mW{{\bm{W}}}

\DeclareMathAlphabet{\mathsfit}{\encodingdefault}{\sfdefault}{m}{sl}
\SetMathAlphabet{\mathsfit}{bold}{\encodingdefault}{\sfdefault}{bx}{n}

\def\sC{{\mathbb{C}}}

\def\sR{{\mathbb{R}}}

\def\sZ{{\mathbb{Z}}}

\def\emA{{A}}

\def\emN{{N}}

\newcommand{\sigmoid}{\sigma}

\newtheorem{theorem}{Theorem}

\cvprfinalcopy 

\def\cvprPaperID{5639} % *** Enter the CVPR Paper ID here

\definecolor{orange}{RGB}{255,127,0}
\begin{document}

\title{Explore Faster Localization Learning For Scene Text Detection}

\author{
{\large
Yuzhong Zhao$^1$},
{\large
Yuanqiang Cai$^2$},
{\large
Weijia Wu$^3$},
{\large
Weiqiang Wang$^1$$^\dagger$}
}

\maketitle
\begin{abstract}

Generally pre-training and long-time training computation are necessary for obtaining a good-performance text detector based on deep networks. In this paper, we present a new scene text detection network (called FANet) with a Fast convergence speed and Accurate text localization. The proposed FANet is an end-to-end text detector based on transformer feature learning and normalized Fourier descriptor modeling, where the Fourier Descriptor Proposal Network and Iterative Text Decoding Network are designed to efficiently and accurately identify text proposals. Additionally, a Dense Matching Strategy and a well-designed loss function are also proposed for optimizing the network performance. Extensive experiments are carried out to demonstrate that the proposed FANet can achieve the SOTA performance with fewer training epochs and no pre-training. When we introduce additional data for pre-training, the proposed FANet can achieve SOTA performance on MSRA-TD500, CTW1500 and TotalText. The ablation experiments also verify the effectiveness of our contributions.

\end{abstract}

\blfootnote{$^1$ University of Chinese Academy of Sciences, China.}
\blfootnote{$^2$ Beijing University of Posts and Telecommunications, China.}
\blfootnote{$^3$ Zhejiang Unversity, China.}
\blfootnote{$^\dagger$ Corresponding author. (wqwang@ucas.ac.cn)}

%%%%%%%%% BODY TEXT
\section{Introduction}

Scene text detection is an important task of computer vision, and a basis of various text-related applications, so many researchers are concerned about the issue~\cite{DBLP:journals/pami/YeD15,DBLP:journals/fcsc/ZhuYB16,DBLP:journals/ijdar/LiuMP19,wu2020texts,wu2020synthetic}. The rise and wide application of deep learning make great progress in scene text detection, so the best performance on benchmark datasets is refreshed constantly. 
However, most of the state-of-the-art methods~\cite{DBLP:conf/iccv/WangXSZWLYS19,DBLP:conf/cvpr/RaisiN0WZ21,DBLP:conf/cvpr/ZhuCLKJZ21,DBLP:conf/iccv/HeGDG17,wu2019textcohesion,wu2020selftext} rely on long-time training to achieve a good performance. 
Generally researchers use related large datasets for \textit{long-time pre-training}, and then finetune the network on the target dataset, or directly carry out \textit{long-time training} on the target dataset. These approaches are not suitable for scenarios that require to rapidly generate models or no large dataset for pre-training.

\begin{figure}[t]
 \centering  %width=0.7\linewidth
 \includegraphics[width=3.3in,height=2.6in]{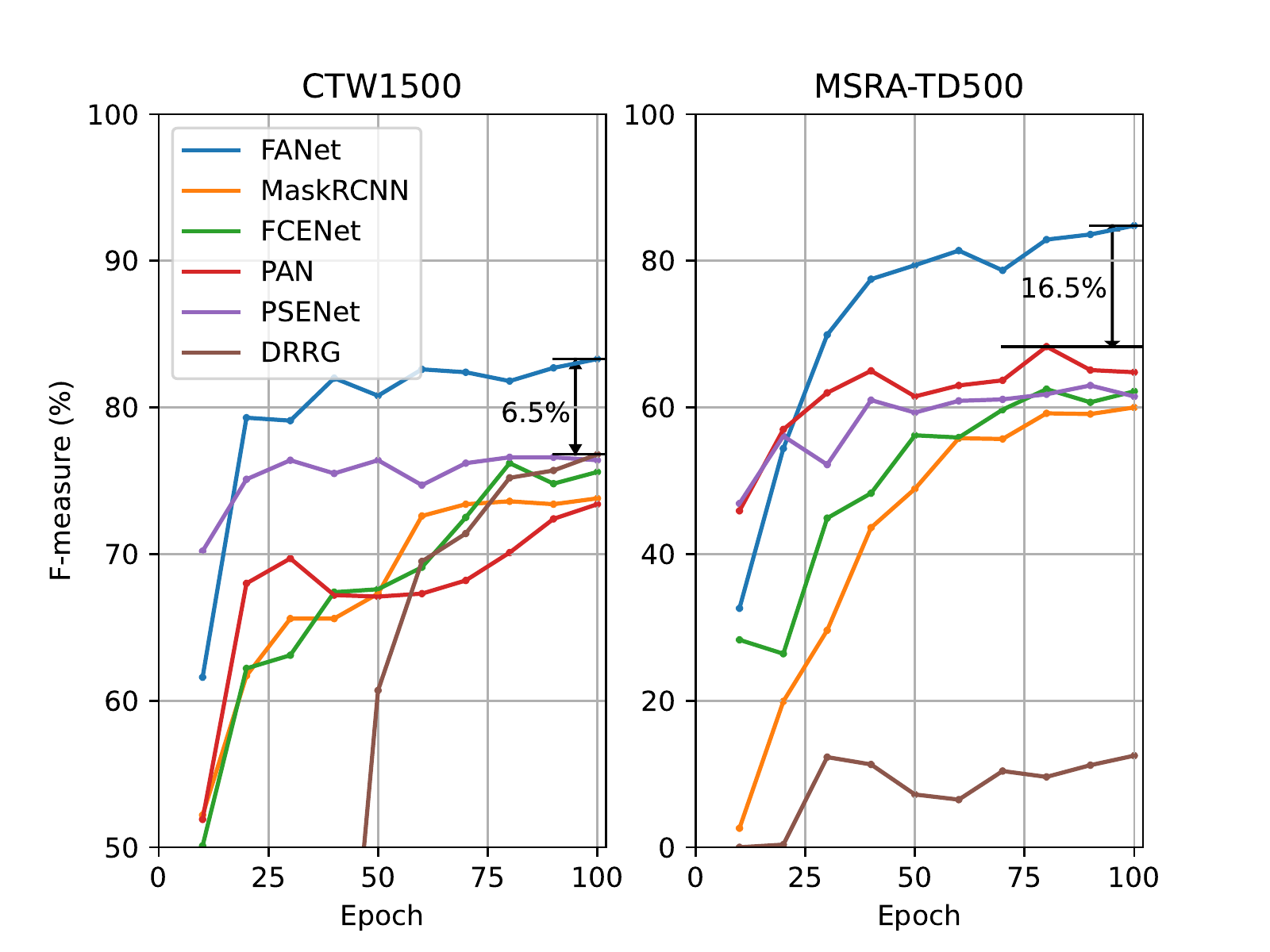}
 \caption{The comparison of convergence performance between the proposed FANet and other state-of-the-art (SOTA) methods. After the training of 100 epochs, the FANet outperforms current best SOTA method by $6.5\%$ $(83.3\% vs. 76.8\%)$ and $16.5\%$ $(84.8\% vs. 68.3\%)$ based on F-measure on datasets CTW1500 and MSRA-TD500 respectively.}
\label{fig:convergence}
%\vspace{-2mm}
\end{figure}

\begin{figure*}[htbp]
\label{fig:arc}
 \centering  %width=0.7\linewidth
 \includegraphics[width=7.in]{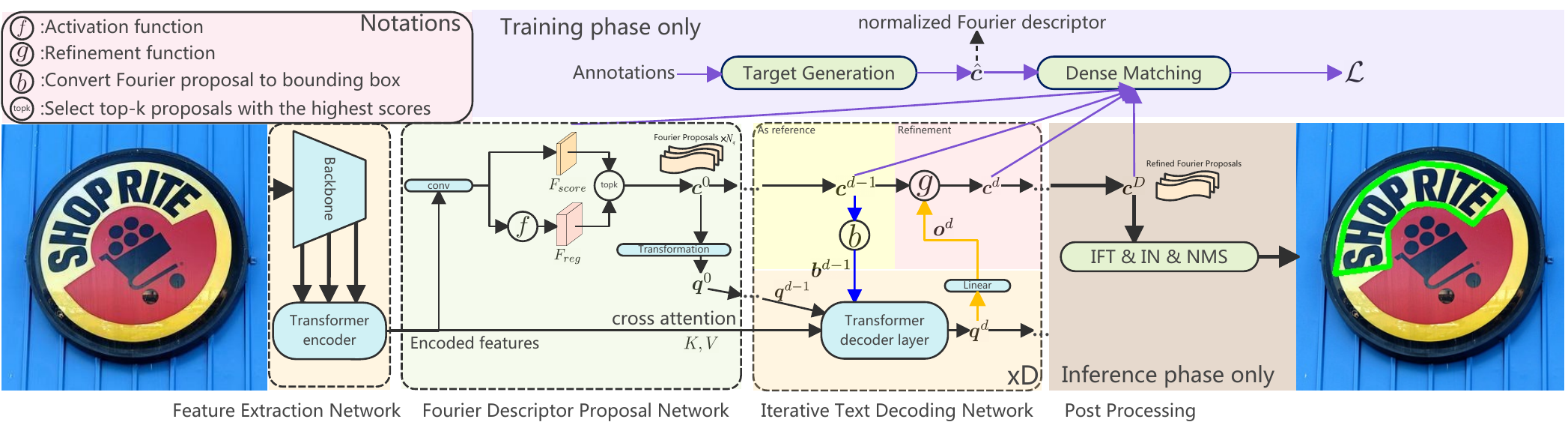}
 \caption{Overview of the proposed FANet. $\vc^d$ denotes the Fourier proposal predicted by $d$-th transformer decoder layer, $\vq^{d-1}$ denote the input query for $d$-th transformer decoder layer, $d=1,\cdots,D$, where $D$ is the number of decoder layers (6 in our experiments). As some special cases, $\vc^0$ denotes the Fourier proposal predicted by Fourier Descriptor Proposal Network (FDPN), $\vq^{D}$ is the encoded memory of the last transformer decoder layer. We only describe the calculation process of a single proposal $\vc^{D}$ in the network for simplicity. In the Iterative Text Decoding Network (ITDN), the {\color{blue}blue} top-down arrow indicates that the Fourier descriptor is used as the reference location for the Multi-Scale Deformable Attention module, the {\color{orange}orange} bottom-up arrow indicates that the offset predicted by the transformer decoder layer is used to refine the Fourier proposals.}
\label{fig:arc}
%\vspace{-2mm}
\end{figure*}

Inspired by the feature learning ability of the transformer \cite{DBLP:conf/nips/VaswaniSPUJGKP17} and deformable DETR\cite{DBLP:conf/iclr/ZhuSLLWD21} and the representation ability of Fourier descriptor for arbitrary contours~\cite{DBLP:conf/cvpr/ZhuCLKJZ21}, we present a scene text detection network (called FANet) with a Fast convergence speed and Accurate text localization, where the transformer is combined with Fourier descriptor to localize arbitrary-shaped text regions. The DETR based metheds~\cite{DBLP:conf/eccv/CarionMSUKZ20} can scale-invariantly localize the rectangular target regions. The Fourier descriptor proposed by ~\cite{DBLP:conf/cvpr/ZhuCLKJZ21} can model target regions of arbitrary contours, but it cannot be embedded into the DETR detection framework. Thus, we propose a normalization method for Fourier descriptor to enable the DETR based methods to predict the normalized text regions of arbitrary shape. However, we find that the established framework based on the proposed normalization method has slow convergence and low accuracy due to two aspects of reasons, i.e., (1) \textit{the change of regression target} makes the deformable DETR component become sub-optimal and (2) the \textit{Hungarian Matching Strategy} hinders the rapid convergence of the network.
Correspondingly, we first make some effective changes to the structure of deformable DETR. Concretely, we propose a Fourier Descriptor Proposal Network (FDPN) to get better  candidates for the text decoder. Then, we build an Iterative Text Decoding Network (ITDN) to iteratively refine Fourier proposals. Finally, we propose a well-designed loss function to optimize the descriptor representations and calculate the matching cost. Additionally, we propose a Dense Matching Strategy (DMS) to greatly speed up the convergence and improve the 
detection accuracy within fewer training epochs.
%Together with the above components and strategies, we improve the performance of our network to a new level. 
% We test the performance of our network on several benchmark datasets with limited training epochs, and achieve the state-of-the-art (SOTA) performance without large-scale dataset pre-training.
As shown in Figure~\ref{fig:convergence}, the proposed FANet can obtain an F-measure of $83.3\%$ and $84.8\%$ respectively, after training only 100 epochs on datasets CTW1500~\cite{DBLP:journals/pr/LiuJZLZ19} and MSRA-TD500~\cite{DBLP:conf/cvpr/YaoBLMT12} without pre-training, which outperforms current best SOTA methods~\cite{DBLP:conf/cvpr/ZhangZHLYWY20,DBLP:conf/iccv/WangXSZWLYS19} by $6.5\%$ $(83.3\%vs.76.8\%)$ and $16.5\%$ $(84.8\%vs.68.3\%)$ respectively. The main contributions of this paper are summarized as follows.  
\begin{itemize}
    % \item We design a progressive optimization strategy, it can guide each block from the complete learning of text instance to the offset learning between adjacent layers in the decoding stage, to reduce the optimization difficult and accelerate the convergence.
    % \item We develop a Fourier parameter proposal head in the encoding stage, which can provide a sampling proposal space for the decoder and guide the optimization direction by reducing the randomness of its parameter space.
    \item We present a scene text detection network FANet with fast convergence speed and accurate localization, which uses the transformer to learn text features, and the normalized Fourier descriptor to represent text regions. The proposed FANet has achieved SOTA performances on multiple public benchmarks,$e.g.$, MSRA-TD500,CTW1500 and TotalText.
    % \item We propose a normalize method for the Fourier coefficient, which enables transformer to predict text instances of arbitrary shape in a regression manner.
    \item We make many changes to the original deformable DETR, including Fourier Descriptor Proposal Network (FDPN), Iterative Text Decoding Network (ITDN) and well-designed loss function. These components can effectvely improve the convergence speed and accuracy of FANet.
    \item We propose a Dense Matching Strategy (DMS), which significantly improves the convergence speed and accuracy of FANet within fewer training epochs.
\end{itemize}

\section{Related Works}
\subsection{Transformer modeling}
The transformer~\cite{DBLP:conf/nips/VaswaniSPUJGKP17} with both self-attention and cross-attention mechanism has achieved great success in both machine translation~\cite{DBLP:conf/naacl/DevlinCLT19} and visual recognition~\cite{DBLP:conf/iclr/DosovitskiyB0WZ21}. For example, DETR \cite{DBLP:conf/eccv/CarionMSUKZ20} first adopts the transformer architecture for the object detection task. deformable DETR \cite{DBLP:conf/iclr/ZhuSLLWD21} extends DETR with a deformable attention module that reduces the training time significantly.
Some previous works also try to explore the potential of transformer on text spotting task.
Wu \emph{et al.}~\cite{wu2021bilingual,wu2022end} proposed to track and spot text in video with transformer sequence modeling.
\cite{DBLP:conf/cvpr/RaisiN0WZ21} adopts the transformer architecture in multi-orientation text detection for the first time, but it still suffers from the problems of requiring massive data for pre-training, slow convergence, poor performance and inadequate representation ability. By using better contour representation, feature extraction and network optimization methods, we make the proposed FANet based on transformer surpasses the SOTA text detection algorithm based on CNNs~\cite{DBLP:conf/cvpr/ZhuCLKJZ21,DBLP:conf/cvpr/DaiZ0C21,DBLP:conf/iccv/WangXSZWLYS19}.

\subsection{Text Region Representation}
Text regions can be modeled via per-pixel masks ~\cite{DBLP:conf/cvpr/WangXLHLY019,DBLP:conf/iccv/WangXSZWLYS19}, or modeled by parameters in specified representation spaces. For example, TextRay \cite{DBLP:conf/mm/WangCW020} represents the text contours in the polar system. ABCNet \cite{DBLP:conf/cvpr/LiuCSHJW20} introduced Bezier curves to parameterize curved texts. FCENet \cite{DBLP:conf/cvpr/ZhuCLKJZ21} represents the text instances in the Fourier domain, which allows to represent any closed continuous contour in robust and simple manners. In this paper, we further present a new normalized Fourier descriptor to represent the normalized text regions of arbitrary shapes, which makes it possible to embed the text representation based on Fourier descriptor into the detection framework based on transformers.

\section{Method Description}
\subsection{Overview}
As shown in Figure~\ref{fig:arc}, the proposed FANet mainly consists of three parts: Feature Extraction Network (FEN), Fourier Descriptor Proposal Network (FDPN) and Iterative Text Decoding Network (ITDN). For a given image, it is first encoded as features by the FEN, which consists of a backbone and a transformer encoder. The encoded features are then fed into the FDPN to obtain a set of arbitrary-shaped text contours represented as normalized Fourier descriptor and the object queries transformed from them. Further, the encoded features, the object queries and the $N_q$ Fourier proposals with the highest scores are jointly fed into the ITDN to obtain the refined Fourier proposals, where $N_q$ is the number of queries of the transformer decoder (300 in our experiments). Finally, we get the detection results after applying Inverse Fourier transform (IFT), Inverse Normalization (IN) and Non-Maximum Suppression (NMS) to the refined Fourier proposals.

\subsection{Fourier Descriptor Normalization}

\begin{figure}[t]
\label{fig:normalization}
 \centering  %width=0.7\linewidth
 \includegraphics[width=3.in]{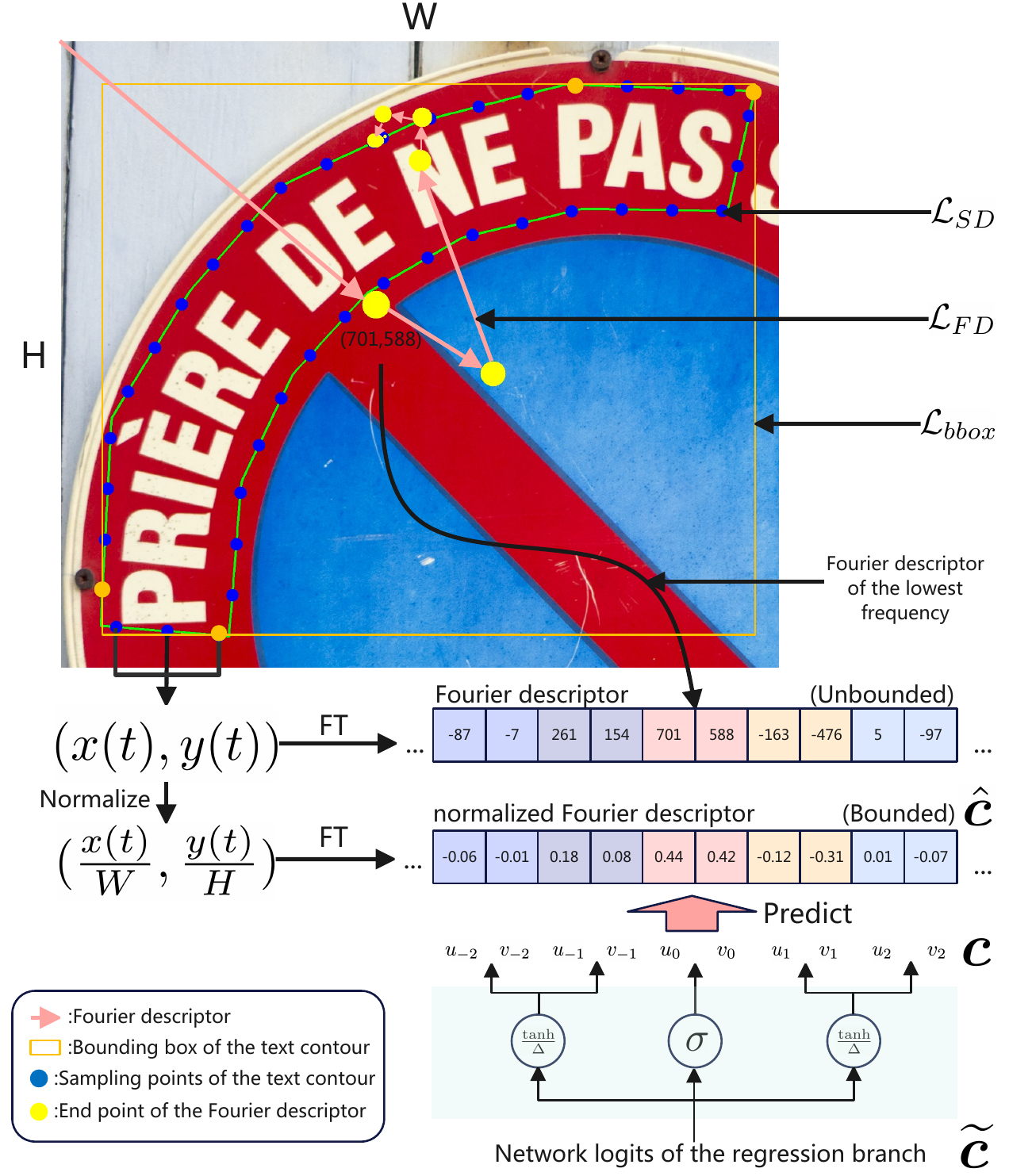}
 \caption{Description of the Fourier descriptor normalization.}
\label{fig:normalization}
%\vspace{-2mm}
\end{figure}

\quad \textbf{Target generation.} FCENet uses a complex-value function $z(t)=x(t)+iy(t)$, $t\in [0,1]$ to represent a text contour, where $(x(t),y(t))$ denotes the spatial coordinate of the text contour. To make the contour representation scale-invariant, we present a normalized form  $\mathcal{Z}(t)=\frac{x(t)}{W}+i\frac{y(t)}{H}=\mathcal{X}(t)+i\mathcal{Y}(t)$, where $H,W$ are the height and width of a given image. Since we cannot obtain the analytical form of text contour function $\mathcal{Z}$ in real scenario, we can discretize the continual function $\mathcal{Z}$ into $N$ points, i.e., $\{\mathcal{Z}(\frac{n}{N})\}$, $n\in\{0,...,N-1\}$. Correspondingly, we can calculate the Fourier descriptor of a normalized text contour by Discrete Fourier Transform (DFT), $i.e.$, 
\begin{align}
    \hat{c}_k&=\frac{1}{N}\sum_{n=0}^{N-1} \mathcal{Z}(\frac{n}{N})e^{-2\pi i k \frac{n}{N}},  k\in \Psi\label{eq:fouier},
\end{align}

where $\hat{c}_k=\hat{u}_k+i\hat{v}_k\in \sC$,  $\hat{u}_k$ and $\hat{v}_k$ are real part and  imaginary part of $\hat{c}_k$, $\Psi \triangleq \{-K,\cdots,0,\cdots,K\}$, $K$ is the highest frequency index reserved for the Fourier descriptor of the text contour ($K$=5 in our experiments). Thus, each ground-truth text contour can be represented as a normalized Fourier descriptor, $\hat{\vc}=[\hat{u}_{-K},\hat{v}_{-K},\cdots,\hat{u}_0,\hat{v}_0,\cdots,\hat{u}_{K},\hat{v}_{K}]\in R^{4K+2}$. We take normalized Fourier descriptor as regression targets. As shown in Figure~\ref{fig:normalization}, the values of original Fourier descriptor are generally large and unbounded, while the values of normalized Fourier descriptor all fall into a small region with fixed boundaries, $i.e.$, \begin{align}
    (\hat{u}_0,\hat{v}_0) &\in [0,1]^2,\\
    (\hat{u}_k,\hat{v}_k) &\in [-\frac{2}{\pi},\frac{2}{\pi}]^2,k\in \Psi,k\neq 0.
\end{align}
% The related proof is provided in supplementary material $\textbf{A}$.\\

\textbf{ New activation function.} Considering the above properties of normalized Fourier descriptor, the commonly used activation functions cannot meet our requirements. For example, the identity function is an unbounded function that does not fit the bounded regression target well. The sigmoid function $\sigmoid$ cannot output negative values. The range of $\tanh$ is much larger than our regression target, which results in too many invalid predictions. So we define a new activation function $f(\widetilde{\vc})$ suitable for the regression of proposed normalized Fourier descriptor as
\begin{align}
    \vc_i = f(\widetilde{\vc}_i)=\left\{
    \begin{array}{cc}
       \sigmoid(\widetilde{\vc}_i)  & i=2K,2K+1, \\
       \frac{\tanh(\widetilde{\vc}_i)}{\Delta} & i\in\Omega/\{2K,2K+1\},\\
    \end{array}\right.
\end{align}
where $\widetilde{\vc}\in \sR^{4K+2}$ is the stimulus, and $\vc\in \sR^{4K+2}$ is the prediction of regression branches, where $\Omega = \{0,\cdots,4K+1\}$ is the index set aligned with $\Psi$, $\Delta$ is a hyperparameter that controls the output range of function $f(\widetilde{\vc})$ and we set it to $\frac{\pi}{2}$ to best match the range of our regression target. As shown in Figure~\ref{fig:normalization}, we use the proposed activation function for the unbounded network logits of the regression branches where we need to predict the normalized Fourier descriptor by the networks, such as FDPN and ITDN.

\subsection{Fourier Descriptor Proposal Network}

We build a lightweight Fourier Descriptor Proposal Network (FDPN) following the Feature Extraction Network. As shown in Figure~\ref{fig:arc}, through a single convolution module, we predict two feature maps, $i.e.$, $F_{score}\in \sR^{B\times 2\times h_f\times w_f},F_{reg}\in \sR^{B\times (4K+2)\times h_f\times w_f}$, where $B$ is the batch size, $h_f$ and $w_f$ are the height and width of the feature map.
Each pixel on the feature map is assigned as an object query, which directly predicts a normalized Fourier descriptor and a corresponding score, $\textit{i.e.}$, a Fourier proposal. The top $N_q$ Fourier proposals with the highest scores are selected, and we convert them into the initial queries of the ITDN through a transformation layer, which consists of a linear layer and a position encoding layer as described in deformable DETR~\cite{DBLP:conf/iclr/ZhuSLLWD21}. Finally, the proposals and the queries are sent to the ITDN, and $N_q$ is also the number of queries of the transformer decoder.

\subsection{Iterative Text Decoding Network}
% \subsection{Iterative Fourier Descriptor Refinement}
As shown in Figure~\ref{fig:arc}, Iterative Text Decoding Network (ITDN) is composed of $D$ duplicate modules, each module is consist of a deformable transformer decoder layer and a refinement layer. We let each module refine the Fourier descriptor based on the prediction from the previous module and we select the output of the last module as the final predictions. The refinement target of FANet is the normalized Fourier descriptor while the refinement target of deformable DETR~\cite{DBLP:conf/iclr/ZhuSLLWD21} is normalized bounding boxes. To adapt to this change of regression target, we make some effective changes for the transformer decoder layer as follows: \\
(1) \textbf{Take the Fourier descriptor as the reference location for the Multi-Scale Deformable Attention Module}. 
% We use the fourier coefficient predicted by the previous decoder layer as the reference point of the multi-scale deformable attention module of the current decoder layer.
As shown in Figure~\ref{fig:arc}, in the calculation process of $d$-th transformer decoder layer, we first calculate the bounding box $\vb^{d-1}=(x,y,w,h)\in [0,1]^4$ of the Fourier proposal $\vc^{d-1}$ through a function $b$. Then we use $((p_x+x)w,(p_y+y)h)$ instead of $(p_x,p_y)$ as the sampling location for the Multi-Scale Deformable Attention (MSDeformAttn) module~\cite{DBLP:conf/iclr/ZhuSLLWD21}, where $(p_x,p_y)$ are the sampling offset predicted by the MSDeformAttn module. Such modifications make the sampled locations of the MSDeformAttn module related to the center and size of previously predicted text contour.\\
% We first convert the proposals of layer $i$ into the form of horizontal rectangular boxs, and then we input the horizontal rectangular boxs of layer $i$ and $Query_i$ to $Decoder_i$, in which the input horizontal rectangular boxs is used as the reference boxs of sampling points in the deformable multi-head cross attention module in the $Decoder_i$, as we believe that collecting information around proposals is often more effective than random collection.\\
(2) \textbf{Use multiplication based refinement instead of the original addition based refinement for the nondc component of Fourier descriptor}. As shown in Figure~\ref{fig:arc}, the refinement function maps a proposal $\vc^{d-1}\in \sR^{4K+2}$ and an offset prediction $\vo^{d}\in \sR^{4K+2}$ to a new proposal $\vc^{d}\in \sR^{4K+2}$, which can be formulated as $\vc^{d}_i = g(\vc^{d-1}_i,\vo^{d}_i),i\in \Omega\label{eq:refine}$, where $\vc^{d-1}$ and $\vc^{d}$ denote the prediction of the regression branch of $(d-1)$-th and $d$-th transformer decoder layer ($\vc_{0}$ denotes the proposal output by FDPN), $\vo^{d}$ is the offset prediction of the $d$-th transformer decoder layer. $\text{refinement}$ function used by deformable DETR is as follows:

\begin{align}
    \vc^{d}_i &= g_1(\vc^{d-1}_i,\vo^{d}_i) = f(f^{-1}(\vc^{d-1}_i)+\vo^{d}_i),i\in \Omega\label{eq:addrefine},
\end{align}
where $f$ and $f^{-1}$ are the activation function and its reverse function. We propose a new multiplication based $\text{refinement}$ function for the nondc component of the normalized Fourier descriptor as
\begin{align}
    \vc^{d}_i &= g_2(\vc^{d-1}_i,\vo^{d}_i) = f(f^{-1}(\vc^{d-1}_i)e^{\vo^{d}_i}),\notag\\
    &i\in \Omega/\{2K,2K+1\}\label{eq:mulrefine}.
\end{align}
The reason why we use such $\text{refinement}$ function are summarized. If we assume that $\vo^{d}$ close to $\vec{\textbf{0}}$, which is reasonable because $\vo^{d}$ is the residual term of the prediction. Based on this assumption, we can prove that for the $\text{refinement}$ function based on addition, we have
\begin{align}
    \lim_{\vo^{d}_i\rightarrow 0}\left|\frac{\partial g_1}{\partial \vo^{d}_i}
    \right| &= \left|\frac{\partial f(z)}{\partial z}|_{z=f^{-1}(\vc^{d-1}_i)}\right|, i\in \Omega,
\end{align}
and for the $\text{refinement}$ function based on multiplication, we have
\begin{align}
    \lim_{\vo^{d}_i\rightarrow 0}\left|\frac{\partial g_2}{\partial \vo^{d}_i}\right| &= \left|\frac{\partial f(z)}{\partial z}|_{z=f^{-1}(\vc^{d-1}_i)}\right|\left|f^{-1}(\vc^{d-1}_i)\right|,\notag\\
    &i\in \Omega/\{2K,2K+1\}.
\end{align}
% \resizebox{.91\linewidth}{!}{$
%     \displaystyle
%     x = \prod_{i=1}^n \sum_{j=1}^n j_i + \prod_{i=1}^n \sum_{j=1}^n i_j + \prod_{i=1}^n \sum_{j=1}^n j_i + \prod_{i=1}^n \sum_{j=1}^n i_j + \prod_{i=1}^n \sum_{j=1}^n j_i
% $}
As we can see, when calculating the gradient of $\vo^{d}_i$, our multiplication based $\text{refinement}$ function has an extra $\left|f^{-1}(\vc^{d-1}_i)\right|$ term compared with the original definition. It is easy to prove that the function $\left|f^{-1}(x)\right|$ is a monotonically increasing function of $|x|$ when $f$ is $\sigmoid$ or $\tanh$. We can deduce that $\left|f^{-1}(\vc^{d-1}_i)\right|$ term can adaptively increase the gradient of the Fourier descriptor if $|\vc^{d-1}_i|$ becomes larger. This weighting scheme is intuitive and effective, because the numerically larger Fourier descriptor generally plays a leading role in the construction of the target contour, which means they should occupy a larger gradient in back propagation than others. As a special case, we still use the $\text{refinement}$ function based on addition for the dc component $(u_0,v_0)$ of Fourier descriptor. $(u_0,v_0)$ is the center coordinate of the text contour and the optimization of the center coordinate of different targets should not be weighted by the absolute value of their position. In the ITDN, the refinement function is defined as 
\begin{align}
        g(\vc^{d}_i)=\left\{
    \begin{array}{cc}
       g_1(\vc^{d-1}_i,\vo^{d}_i)  & i=2K,2K+1, \\
       g_2(\vc^{d-1}_i,\vo^{d}_i) & i\in\Omega/\{2K,2K+1\}.\\
    \end{array}\right.
\end{align}

\subsection{Optimization strategies}
\quad \textbf{Dense Matching Strategy.} The traditional Hungarian Matching Strategy (HMS) used in \cite{DBLP:conf/cvpr/StewartAN16, DBLP:conf/eccv/CarionMSUKZ20, DBLP:conf/iclr/ZhuSLLWD21} only matches one query for each ground-truth text instance, which we find is one of the causes of slow convergence of the network. The number of text instances in an image is usually limited,and we denote it as $N_g$. $N_g$ is usually much less than the number of queries of the network $N_q$ (300 in our experiments), $\textit{i.e.}$, $N_g\ll N_q$. Using HMS means that only $N_g$ queries are used as positive samples for the training of regression branch in each iteration, which makes the network easy to overfit and damages the performance of the network due to lack of positive samples. We propose to match multiple queries for each ground-truth text instance. Dense Matching Strategy (DMS) can effectively alleviate the problem of slow convergence, especially in the early stage of training. Since the DMS makes the FANet output overlapping detection results, we have to use the NMS as post-processing. Due to the small number of queries of DETR based methods and small number of $N_m$, the number of predictions is generally small, so the NMS only brings a low cost.
The pseudo code of Dense Matching Strategy is summarized:

\begin{algorithm}[htbp]
\small
\caption{Dense Matching Strategy}
\label{assign}
\raggedright
\textbf{Input}: \\
$\hat{\mathcal{A}}$ : a set of ground-truth text contours in a given image.\\
$\mathcal{A}$ : a set of predict text contours in the image.\\
\textbf{Parameter}: \\
$\text{Cost}$ : the cost function of the proposed FANet.\\
$N_m$ : the number of predicted text contour matched for each ground truth text contour.\\
\textbf{Output}: \\
$\mathcal{P}$ : a set of positive samples.\\
$\mathcal{N}$ : a set of negtive samples.
\begin{algorithmic}[1] %[1] enables line numbers
\STATE $\mathcal{P}\leftarrow \varnothing$
\STATE $\hat{a}=|\hat{\mathcal{A}}|$
\STATE $a=|\mathcal{A}|$
\STATE $i=1$
\STATE $cost = \text{Cost}(\mathcal{A},\hat{\mathcal{A}})\in \sR^{a\times\hat{a}}$
\WHILE{$i\leq N_m$}
\STATE $k,l = \text{Hungarian\_matcher}(cost)$
\STATE $\mathcal{P}\leftarrow \mathcal{A}[k]$
\STATE $cost[k,:]=+\infty$
\STATE $i = i+1$
\ENDWHILE
\STATE $\mathcal{N}=\mathcal{A}-\mathcal{P}$
\RETURN $\mathcal{P},\mathcal{N}$
\end{algorithmic}
\end{algorithm}

To use the DMS algorithm, we need to define a \text{Cost} function that takes the set of predicted text contours $\mathcal{A}$ and the set of ground-truth text contours $\hat{\mathcal{A}}$ as input and outputs the cost between any two elements in the two sets. We use $\mathcal{C}_{cls}+\lambda\mathcal{C}_{reg}$ as the matching cost for any $(\vc,\hat{\vc})$ pair in the two sets, where $\mathcal{C}_{cls} = -log(s)$, where $s$ is the predicted text confidence score corresponding to proposal $\vc$, $\lambda$ is the hyperparameter that balance the costs. We define $\mathcal{C}_{reg}$ as follows:
\begin{align}
    \mathcal{C}_{reg}(\vc,\hat{\vc})&=\mathcal{L}_{SD}(\vc,\hat{\vc})+\alpha_1\mathcal{L}_{FD}(\vc,\hat{\vc})+\alpha_2\mathcal{L}_{bbox}(\vc,\hat{\vc})\label{eq:regloss},\\
    \mathcal{L}_{SD}(\vc,\hat{\vc})&=\text{L1}\big[\mathcal{F}^{-1}(\vc),\mathcal{F}^{-1}(\hat{\vc})\big],\\
    \mathcal{L}_{FD}(\vc,\hat{\vc})&=\text{L1}(\vc,\hat{\vc}),\\
    \mathcal{L}_{bbox}(\vc,\hat{\vc})&=\text{GIOU}\big[b(\vc),b(\hat{\vc})\big].
\end{align}
Where $\alpha_1,\alpha_2$ are the hyperparameters that balance the three matching costs. $\text{L1}$ is the L1 loss function, $b$ is a function that convert Fourier proposal to bounding box, $\text{GIOU}$ is the $\text{GIOU}$ loss in~\cite{DBLP:conf/cvpr/RezatofighiTGS019}, $\mathcal{F}^{-1}$ is the Inverse Discrete Fourier Transform. For a given image, we first calculate the $cost$ between $\mathcal{A}$ and $\hat{\mathcal{A}}$, then we iteratively perform Hungarian Matching (HM) for $N_m$ times. Each time we first add the proposals that match any ground-truth text contour to the positive sample set $\mathcal{P}$ and then remove these proposals from the matching queue by setting the matched rows of $cost$ to $+\infty$ to avoid repeated matching. Finally, we collect all the matched proposals as positive samples and the unmatched proposals as negative samples. 

\textbf{Loss function.}
The loss function of the proposed network is given by
\begin{align}
    \mathcal{L}&=\mathcal{L}_{cls,0}+\lambda\mathcal{L}_{reg,0}+\sum_{i=1}^{D}w_i(\mathcal{L}_{cls,i}+\lambda\mathcal{L}_{reg,i}),
\end{align}
where $\mathcal{L}_{cls,0}$ and $\mathcal{L}_{reg,0}$ denote the classification and regression loss of the FDPN respectively, $\mathcal{L}_{cls,i}$ and $\mathcal{L}_{reg,i}$ denote the classification and regression loss of the $i$-th decoder layer, $w_i$ is the hyperparameter that balance the losses of different transformer decoder layers, $D$ is the number of decoder layers. We use focal loss~\cite{DBLP:conf/iccv/LinGGHD17} as our default classification loss for $\mathcal{L}_{cls,i}, i=0,1,\cdots,D$. We define $\mathcal{L}_{reg,i}$ as follows:

\begin{align}
    \mathcal{L}_{reg,i}&=\frac{1}{|\mathcal{M}_i|}\sum_{(\vc, \hat{\vc}) \in \mathcal{M}_i}(\mathcal{L}_{SD}+\alpha_1\mathcal{L}_{FD}+\alpha_2\mathcal{L}_{bbox})\label{eq:regloss},
\end{align}

where $\mathcal{M}_i, i=0,1,\cdots,d$ are the set of matched Fourier descriptor pairs of $i$-th decoder layer, $i.e.$, $\forall (\vc,\hat{\vc}) \in \mathcal{M}_i, \vc \in \mathcal{P}_i$, where $\mathcal{P}_i$ is the positive sample set obtained by DMS at the $i$-th decoding layer.

\begin{table*}[htbp]
\setlength{\tabcolsep}{2.8pt}
\small
\centering
\caption{Comparison with recent state-of-the-art methods on ICDAR 2015, MSRA-TD500, CTW1500 and TotalText under the protocol of IoU@0.5, where 'Ext.' denotes extra training data. We use ResNet50 as the default backbone for our proposed FANet and all the compared algorithms.}
\label{tab:performance}
\begin{tabular}{cccc|ccc|ccc|ccc|ccc}
\hline
\multirow{2}{*}{
Methods} & 
\multirow{2}{*}{
Venue} &
\multirow{2}{*}{
Backbone} & 
\multirow{2}{*}{
Ext.} & \multicolumn{3}{c|}{
ICDAR2015} & \multicolumn{3}{c|}{
MSRA-TD500} & \multicolumn{3}{c|}{
CTW1500} & \multicolumn{3}{c}{
TotalText}\tabularnewline
%\cline{5-16}
 &  &  &  &
R(\%) & 
P(\%) & 
F(\%) & 
R(\%) & 
P(\%) & 
F(\%) &
R(\%) &
P(\%) &
F(\%) &
R(\%) &
P(\%) &
F(\%)\tabularnewline
\hline

TextSnake~\cite{DBLP:conf/eccv/LongRZHWY18} & 
ECCV'18 & 
VGG16~\cite{DBLP:journals/corr/SimonyanZ14a} &
\checkmark & 
80.4 & 
84.9 & 
82.6 & 
73.9 & 
83.2 & 
78.3 & 
\textbf{85.3} & 
67.9 & 
75.6 & 
74.5 & 
82.7 & 
78.4\tabularnewline

PAN~\cite{DBLP:conf/iccv/WangXSZWLYS19} & 
ICCV'19 &
&
\checkmark & 
81.9 & 
84.0 & 
82.9 & 
83.8 & 
84.4 & 
84.1 & 
81.2 & 
86.4 & 
83.7 & 
81.0 & 
\textbf{89.3} & 
85.0\tabularnewline

DB~\cite{DBLP:conf/aaai/LiaoWYCB20} & 
AAAI'20 &
Res50-DCN~\cite{DBLP:conf/iccv/DaiQXLZHW17} &
\checkmark & 
83.2 & 
\textbf{91.8} & 
\textbf{87.3} & 
79.2 & 
91.5 & 
84.9 & 
80.2 & 
86.9 & 
83.4 & 
82.5 & 
87.1 & 
84.7\tabularnewline

CRNet~\cite{DBLP:conf/mm/ZhouXFLZ20} & 
MM'20 &
&
\checkmark & 
84.5 & 
88.3 & 
86.4 & 
82.0 & 
86.0 & 
84.0 & 
80.9 & 
87.0 & 
83.8 & 
82.5 & 
85.8 & 
84.1\tabularnewline

\cite{DBLP:conf/cvpr/RaisiN0WZ21} & 
CVPRW'21 &
&
\checkmark & 
78.3 & 
89.8 & 
83.7 & 
83.8 & 
90.9 & 
87.2 & 
- & 
- & 
- & 
- & 
- & 
-\tabularnewline

DRRG~\cite{DBLP:conf/cvpr/ZhangZHLYWY20} & 
CVPR'20 &
&
\checkmark & 
84.7 & 
88.5 & 
86.6 & 
82.3 & 
88.1 & 
85.1 & 
83.0 & 
85.9 & 
84.5 & 
\textbf{84.9} & 
86.5 & 
85.7\tabularnewline

PCR~\cite{DBLP:conf/cvpr/DaiZ0C21} & 
CVPR'21 &
DLA34~\cite{DBLP:conf/cvpr/DaiZ0C21} &
\checkmark & 
- & 
- & 
- & 
83.5 & 
90.8 & 
87.0 & 
82.3 & 
87.2 & 
84.7 & 
82.0 & 
88.5 & 
85.2\tabularnewline

PSENet~\cite{DBLP:conf/cvpr/WangXLHLY019} & 
CVPR'19 &
 &
 & 
79.7 & 
81.5 & 
80.6 & 
- & 
- & 
- & 
75.6 & 
80.6 & 
78.0 & 
75.1 & 
81.8 & 
78.3\tabularnewline

PAN~\cite{DBLP:conf/iccv/WangXSZWLYS19} & 
ICCV'19 &
&
 & 
77.8 & 
82.9 & 
80.3 & 
77.3 & 
80.7 & 
78.9 & 
77.7 & 
84.6 & 
81.0 & 
79.4 & 
88.0 & 
83.5\tabularnewline

TextRay~\cite{DBLP:conf/mm/WangCW020} & 
MM'20 &
&
 & 
- & 
- & 
- & 
- & 
- & 
- & 
80.4 & 
82.8 & 
81.6 & 
77.9 & 
83.5 & 
80.6\tabularnewline

PCR~\cite{DBLP:conf/cvpr/DaiZ0C21} & 
CVPR'21 &
&
 & 
- & 
- & 
- & 
77.8 & 
87.6 & 
82.4 & 
79.8 & 
85.3 & 
82.4 & 
80.2 & 
86.1 & 
83.1\tabularnewline

FCENet~\cite{DBLP:conf/cvpr/ZhuCLKJZ21} & 
CVPR'21 &
&
 & 
84.2 & 
85.1 & 
84.6 & 
- & 
- & 
- & 
80.7 & 
85.7 & 
83.1 & 
79.8 & 
87.4 & 
83.4\tabularnewline
\hline

FANet & 
 & & &
83.8 & 
85.6 & 
84.7 & 
83.3 & 
91.7 & 
87.3 & 
84.3 & 
85.6 & 
84.9 & 
83.3 & 
86.2 & 
84.8\tabularnewline
FANet & & &
\checkmark & 
\textbf{87.6} & 
85.0 & 
86.3 & 
\textbf{84.2} & 
\textbf{92.1} & 
\textbf{88.0} & 
84.0 & 
\textbf{87.8} & 
\textbf{85.9} & 
84.7 & 
87.1 & 
\textbf{85.9}\tabularnewline
\hline
\end{tabular}
\end{table*}

\begin{table*}
\small
\centering
\caption{Convergence performance comparison with recent state-of-the-art methods. We use F-measure as the evaluation protocol, and $\{30,100,500\}$e means training $\{30,100,500\}$ epochs on the dataset. The results of other algorithms are reproduced in open-mmocr~\cite{DBLP:conf/mm/KuangS0YLCWZGZC21}. $\dagger$ denote using original Hungarian Matching Strategy (HMS).}
\label{tab:convergence}
\begin{tabular}{c|ccc|ccc|ccc|ccc}
\hline
\multirow{2}{*}{
Methods} & \multicolumn{3}{c|}{
ICDAR2015} & \multicolumn{3}{c|}{
MSRA-TD500} & \multicolumn{3}{c|}{
CTW1500} & \multicolumn{3}{c}{
TotalText}\tabularnewline
%\cline{2-13}
 & 
30e & 
100e & 
500e & 
30e & 
100e & 
500e & 
30e & 
100e & 
500e & 
30e & 
100e & 
500e\tabularnewline
\hline
PAN~\cite{DBLP:conf/iccv/WangXSZWLYS19} & 
62.3 & 
73.8 & 
74.9 & 
61.4 & 
65.1 & 
70.0 & 
72.0 & 
73.4 & 
74.6 & 
66.9 & 
73.7 & 
76.0\tabularnewline
FCENet~\cite{DBLP:conf/cvpr/ZhuCLKJZ21} & 
65.6 & 
78.5 & 
83.5 & 
47.2 & 
62.2 & 
74.2 & 
65.5 & 
75.6 & 
81.0 & 
71.3 & 
80.2 & 
83.0\tabularnewline
\hline
FANet$\dagger$ & 
73.9 & 
82.6 & 
83.5 & 
62.1 & 
70.6 & 
83.5 & 
75.6 & 
81.2 & 
84.6 & 
79.0 & 
81.2 & 
84.6\tabularnewline
FANet & 
\textbf{77.2} & 
\textbf{84.1} & 
\textbf{84.7} & 
\textbf{70.7} & 
\textbf{84.8} & 
\textbf{87.3} & 
\textbf{79.1} & 
\textbf{83.3} & 
\textbf{84.9} & 
\textbf{81.9} & 
\textbf{84.1} & 
\textbf{84.8}\tabularnewline
\hline
\end{tabular}
\end{table*}

\section{Experiment}
% We conducted quantitative experiments on four public benchmarks: ICDAR2015~\cite{DBLP:conf/icdar/KaratzasGNGBIMN15}, MSRA-TD500~\cite{DBLP:conf/cvpr/YaoBLMT12}, CTW1500~\cite{DBLP:journals/pr/LiuJZLZ19} and TotalText~\cite{DBLP:conf/icdar/ChngC17} to verify the effectiveness of the proposed algorithm and the effectiveness of different components.
\subsection{Implementation details}
The backbone of FANet is ResNet-50~\cite{DBLP:conf/cvpr/HeZRS16} which is pre-trained on ImageNet~\cite{DBLP:conf/cvpr/DengDSLL009}. Following FCENet~\cite{DBLP:conf/cvpr/ZhuCLKJZ21}, during the target generation stage, we sample equidistantly a fixed number $N$ ($N$ = 400 in our experiments) points on the text contour, obtaining the resampled point sequence, i.e., $\{\mathcal{Z}(\frac{n}{N})\}$, $n\in\{0,...,N-1\}$. Then, we transform the resampled point sequence into its corresponding normalized Fourier descriptor with Equation~\ref{eq:fouier}. We set the highest level reserved for the Fourier descriptor $K$ to 5. The number of queries of the transformer decoder $N_q$ to 300. The number of transformer layers $D$ to 6. For the hyperparameters that balance the losses, we set $[\lambda,\alpha_1,\alpha_2]$ equals to $[0.25,5,0.4]$ and set $w_d$ to $1$, $d=1,2,\cdots,D$. 

\textbf{Optimization Setting.} The experiments are conducted on the workstation with 8 Tesla V100 GPU. The data augmentation contains color jitter, random rotation, random horizontal flip, random scale, random crop, and random resize. Following the common practices, we ignore the blurred text regions labeled as “DO NOT CARE” during training. The network is trained using an Adam optimizer with the weight decay ratio equals to 0.0001 and we fix the batch size to 16. We use $\textbf{poly}$ as our default learning rate policy, and we set its power as 0.9. For the experiments without pre-training, we set the match number of the DMS $N_m$ to 3 by default and MSRA-TD500 to 10 and train the proposed FANet for 500 epochs separately on the dataset we report the results. For the experiment with pre-training, we first pre-trained our model for 25 epochs on COCOTextv2~\cite{veit2016cocotext}, and then finetune our model for 250 epochs on the benchmark datasets respectively. we use HMS as default and DMS $N_m$ to 5 for MSRA-TD500 by default when FANet is pretrained use COCOTextv2. For the convergence experiments in Table~\ref{tab:convergence}, as the learning rate attenuation strategy will damage the performance of the algorithm under the restriction of very little iterations, we keep the initial learning rate unchanged for the experiments of training 30 epochs. For the experiments of training more than 30 epochs, we keep the optimization settings of the re-implemented algorithms consistent with the original paper.

\textbf{Inference setting.} In the inference stage, we resize the long dimension of test images to 1080, 1280, 2200 and 1312 for MSRA-TD500, CTW1500, ICDAR2015 and TotalText respectively.

\subsection{Comparison with the state-of-the-art methods}
We evaluate the proposed FANet on four public benchmark datasets, $i.e.$,  ICDAR2015~\cite{DBLP:conf/icdar/KaratzasGNGBIMN15}, MSRA-TD500~\cite{DBLP:conf/cvpr/YaoBLMT12}, CTW1500~\cite{DBLP:journals/pr/LiuJZLZ19} and TotalText~\cite{DBLP:conf/icdar/ChngC17}, and the experiments results are summarized in Table \ref{tab:performance}.

\textbf{ICDAR2015.} The proposed FANet achieves comparable performance with SOTA methods on ICDAR2015. It's worth noting that the proposed FANet outperforms the previous DETR based text detection algorithm~\cite{DBLP:conf/cvpr/RaisiN0WZ21} by 2.6\% (86.3\% $vs.$ 83.7\%) based on F-measure with pre-training.

\textbf{MSRA-TD500.} The proposed FANet achieves the SOTA performance on MSRA-TD500. Specifically, if compared with the methods without pre-training, the proposed FANet outperforms the previous SOTA method PCR by 4.9\% (87.3\% $vs.$ 82.4\%) in F-measure. If compared with the algorithm with pre-training, the proposed FANet can still outperforms the previous SOTA method~\cite{DBLP:conf/cvpr/RaisiN0WZ21} by 0.8\% (88.0\% $vs.$ 87.2\%).

\textbf{CTW1500.} The proposed FANet achieves the SOTA performance on CTW1500. if compared with the methods without pre-training, the proposed FANet outperforms the previous SOTA method FCENet by 1.8\% (84.9\% $vs.$ 83.1\%) in F-measure. If compared with the algorithm with pre-training, the proposed FANet can still outperforms the previous SOTA method~\cite{DBLP:conf/cvpr/DaiZ0C21} by 1.2\% (85.9\% $vs.$ 84.7\%).

\textbf{TotalText.} The proposed FANet achieves the SOTA performance on TotalText. If compared with the methods without pre-training, the proposed FANet outperforms the previous SOTA method PAN by 1.3\% (84.8\% $vs.$ 83.5\%) in F-measure. If compared with the algorithm with pre-training, the proposed FANet can still outperforms the previous SOTA method~\cite{DBLP:conf/cvpr/ZhangZHLYWY20} by 0.2\% (85.9\% $vs.$ 85.7\%).

\textbf{Convergence Performance.} Under the constraint of limited epochs, the proposed FANet can still achieve good results. In particular, with only 30 epochs of training, FANet surpass FCENet by 11.6\%, 23.5\%, 13.6\%, 10.6\%, surpass PAN by 14.9\%, 9.3\%, 7.1\%, 15.0\% based on F-measure on ICDAR2015, MSRA-TD500, CTW1500 and TotalText respectively, which shows that the proposed FANet can achieve much better performance than the previous SOTA methods with fewer epochs.
With only 100 epochs of training, FANet can achieve the detection performance of 84.8\% on MSRA-TD500, 83.3\% on CTW1500 and 84.1\% on TotalText, which already have achieved the SOTA performance if only compared with the methods without pre-training.
In other words, we only need to train FANet on a server with 8 Tesla V100 for about 20 minutes, and we can get a SOTA text detector that achieve 84.8\% F-measure on MSRA-TD500. 

\textbf{Qualitative Results.} Some detection results of the proposed FANet are shown in Figure~\ref{fig:results}. For the datasets whose text instances are all rectangular boxes, we take the minimum enclosing rectangle for the text contour. As we can see, FANet is capable of detecting text in a variety of scenarios, including extremely long text in MSRA-TD500, scene text in ICDAR2015, curved text in CTW1500 and TotalText. Qualitative and quantitative results show that the proposed FANet has a wide range of application scenarios.

\begin{figure*}[t]
\label{fig:results}
 \centering  %width=0.7\linewidth
 \includegraphics[width=6.8in]{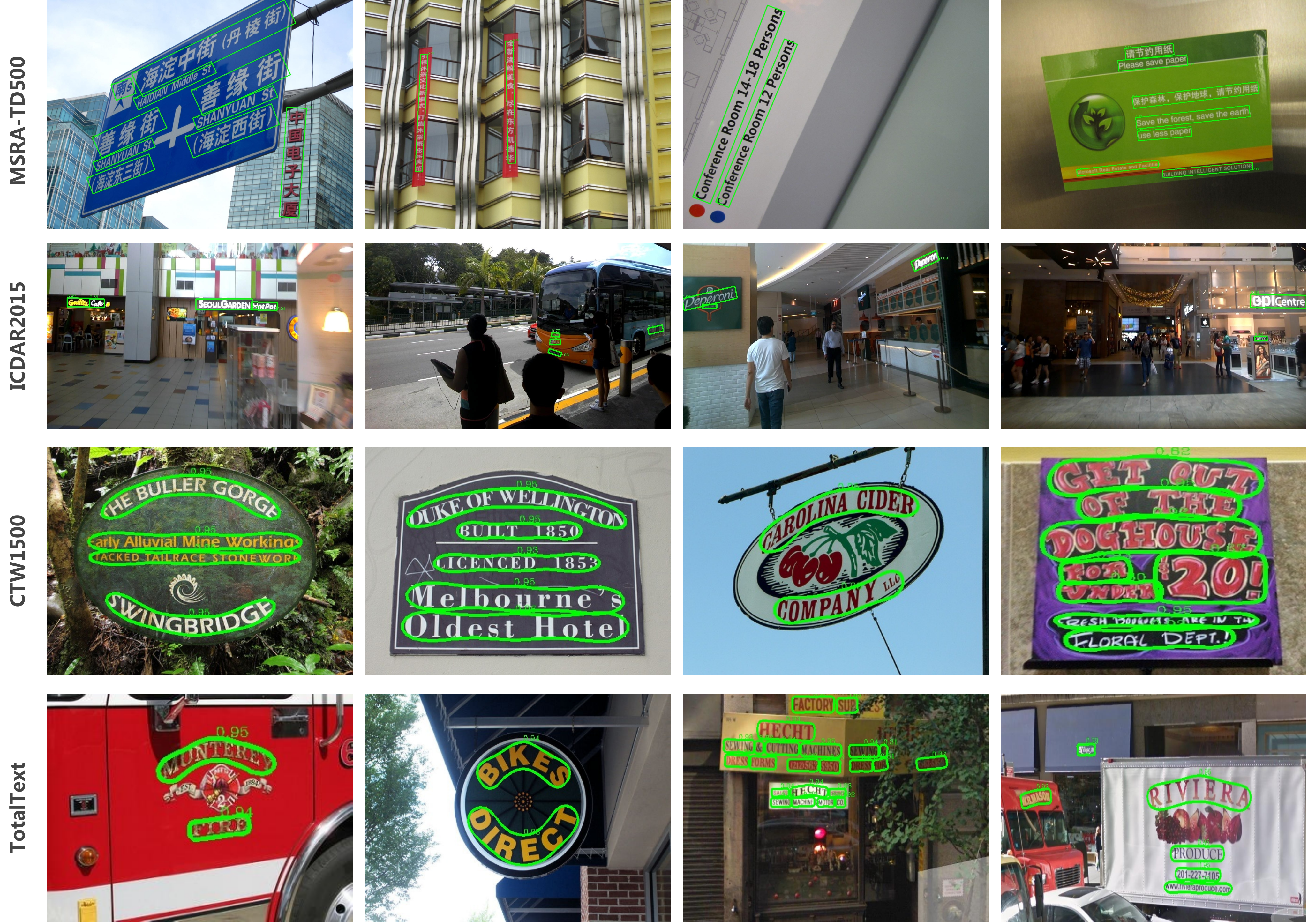}
 \caption{Some detection results of our proposed method on MSRA-TD500, ICDAR2015, CTW1500 and TotalText.}
\label{fig:results}
%\vspace{-2mm}
\end{figure*}

\subsection{Ablation study}
We conduct ablation studies on CTW1500 and MSRA-TD500 under the protocol of IoU@0.5. We fix the number of training epochs for all experiments to 500 epochs.

% ablation of core components
\begin{table}
\setlength{\tabcolsep}{2pt}
\small
\centering
\caption{Ablation experiments on our proposed core components.}
\label{tab:core}
\begin{tabular}{ccc|ccc|ccc}
\hline
\multirow{2}{*}{
ITDN} & \multirow{2}{*}{
FDPN} & \multirow{2}{*}{
DMS} & \multicolumn{3}{c|}{
CTW1500} & \multicolumn{3}{c}{
MSRA-TD500}\tabularnewline
%\cline{4-9}
 &  &  & 
R(\%) & 
P(\%) & 
F(\%) & 
R(\%) & 
P(\%) & 
F(\%)\tabularnewline
\hline

 & 
 & 
 & 
78.9 & 
82.0 & 
80.4 & 
19.1 & 
9.2 & 
12.4\tabularnewline

\checkmark & 
 & 
 & 
79.2 & 
\textbf{87.4} & 
83.1 & 
58.8 & 
67.7 & 
62.9\tabularnewline

\checkmark & 
\checkmark & 
 & 
82.5 & 
86.9 & 
84.6 & 
79.2 & 
88.3 & 
83.5\tabularnewline

\checkmark & 
\checkmark & 
\checkmark & 
\textbf{84.3} & 
85.6 & 
\textbf{84.9} & 
\textbf{83.3} & 
\textbf{91.7} & 
\textbf{87.3} \tabularnewline
\hline
\end{tabular}
\end{table}

\textbf{Baseline and core components.}
Based on deformable DETR, we directly use normalized Fourier descriptor to replace bounding box as the regression target to obtain our baseline algorithm. As shown in Table~\ref{tab:core}, the baseline of the proposed method reaches 80.4\% on CTW1500 and 12.4\% on MSRA-TD500. Compared with our baseline, ITDN can bring relative improvements of 2.7\% (83.1\% $vs.$ 80.4\%) and 50.5\% (62.9\% $vs.$ 12.4\%) on 
CTW1500 and MSRA-TD500 respectively. Then, the addition of FDPN can bring relative improvements of 1.5\% (84.6\% $vs.$ 83.1\%) and 20.6\% (83.5\% $vs.$ 62.9\%) on CTW1500 and MSRA-TD500 respectively. Finally, the DMS can bring relative improvements of 0.3\% (84.9\% $vs.$ 84.6\%) and 3.8\% (87.3\% $vs.$ 83.5\%) on CTW1500 and MSRA-TD500 respectively. The above modules and strategies can significantly improve the performance of the network, especially for MSRA-TD500 which only has small number of training data.

\begin{figure}[htbp]

 \centering  %width=0.7\linewidth
 \includegraphics[width=3.1in]{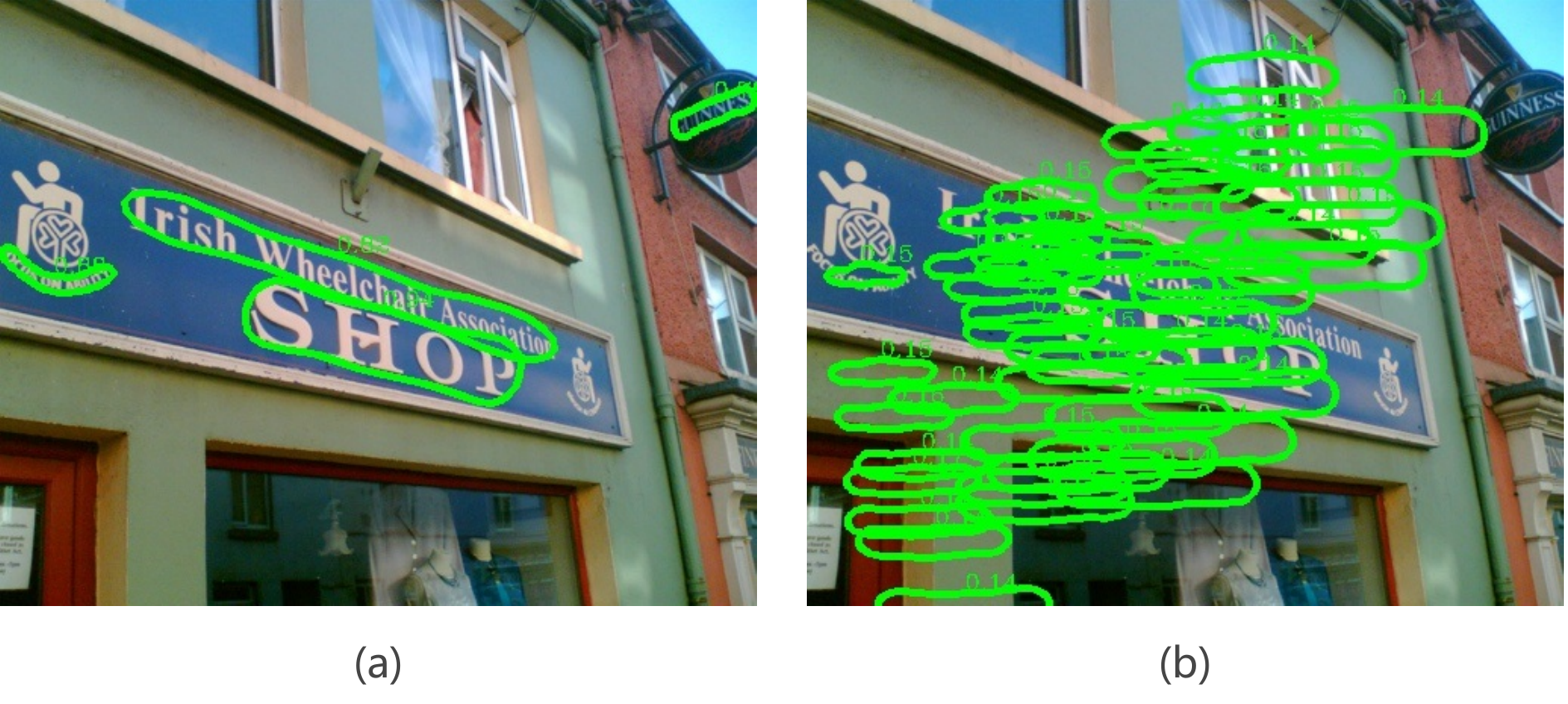}
 \caption{Comparison between normalized Fourier Descriptor regression target and original Fourier Descriptor regression target when DETR like detection architecture is used. (a) denotes using normalized Fourier Descriptor as regression target and our activation function, (b) denotes using original Fourier Descriptor as regression target and identity activation function like FCENet.}
 \label{fig:norm}
%\vspace{-2mm}
\end{figure}

% ablation of normalization
\begin{table}
\setlength{\tabcolsep}{4.0pt}
\small
\centering
\caption{Comparison between our activation function and other commonly used activation functions. AF denotes Activation Function and NT denotes whether to use the Normalized Fourier Descriptor as the regression target. "---" denotes that the network becomes untrainable and no result can be obtained.}
\label{tab:normalization}
\begin{tabular}{cc|ccc|ccc}
\hline
\multirow{2}{*}{
AF} & \multirow{2}{*}{NT} & \multicolumn{3}{c|}{
CTW1500} & \multicolumn{3}{c}{
MSRA-TD500}\tabularnewline
%\cline{3-8}
 & &
R(\%) & 
P(\%) & 
F(\%) & 
R(\%) & 
P(\%) & 
F(\%)\tabularnewline
\hline

identity & 
 &
0.0 & 
0.0 & 
0.0 & 
0.0 & 
0.0 & 
0.0\tabularnewline

identity & 
\checkmark &
--- & 
--- & 
--- & 
--- & 
--- & 
---\tabularnewline

$\sigmoid$ & 
\checkmark &
0.0 & 
0.0 & 
0.0 & 
1.6 & 
1.9 & 
1.7\tabularnewline

$\tanh$ & 
\checkmark &
\textbf{82.9} & 
85.7 & 
84.3 & 
75.8 & 
86.6 & 
80.8\tabularnewline
% \hline 
% 
% ours delta3.14 & 
% 82.21 & 
% 87.38 & 
% 84.72 & 
% 79.55 & 
% 88.70 & 
% 83.88\tabularnewline

ours & 
\checkmark &
82.5 & 
\textbf{86.9} & 
\textbf{84.6} & 
\textbf{79.2} & 
\textbf{88.3} & 
\textbf{83.5} \tabularnewline
\hline
\end{tabular}
\end{table}

\textbf{Activation function.} As shown in Figure~\ref{fig:norm}, we first compare the performance between the normalized Fourier descriptor and the original Fourier descriptor when DETR like architecture is used. The network does not locate text well when the original Fourier descriptor is used which means that the original Fourier descriptor do not fit with DETR like detection architecture.
As shown in line 2 of Table~\ref{tab:normalization}, the network will not be trainable if the activation function in the proposed FANet is replaced by identity function. 
Then we replace the activation function in the proposed FANet with $\sigmoid$ in line 3 of Table~\ref{tab:normalization}, which makes the detection performance decrease to a very low level, since the range of $\sigmoid$ can not completely cover the normalized Fourier descriptor to be predicted. 
Finally, we test the performance when replacing the activation function by $\tanh$ in line 4 of Table~\ref{tab:normalization}, which makes the corresponding F-measure of the network decrease by 0.3\% (84.6\% $vs.$ 84.3\%) and 2.7\% (83.5\% $vs.$ 80.8\%) on CTW1500 and MSRA-TD500 respectively.
The experimental results show that it is more appropriate to use the activation function which can well match the range of regression target.

\begin{table*}[htbp]
\setlength{\tabcolsep}{5.pt}
\small
\centering
\caption{The impact of the match number of the Dense matching Strategy for the performance of the network under the restriction of training 500 epochs. }
\label{tab:match}
\begin{tabular}{c|ccc|ccc|ccc|ccc}
\hline
\multirow{2}{*}{
Match number} & \multicolumn{3}{c|}{
MSRA-TD500} & \multicolumn{3}{c|}{
ICDAR2015} & \multicolumn{3}{c|}{
CTW1500} & \multicolumn{3}{c}{
TotalText}\tabularnewline
%\cline{2-13} 
 & 
R(\%) & 
P(\%) & 
F(\%) & 
R(\%) & 
P(\%) & 
F(\%) & 
R(\%) & 
P(\%) & 
F(\%) & 
R(\%) & 
P(\%) & 
F(\%)\tabularnewline
\hline

1 & 
79.2 & 
88.3 & 
83.5 & 
79.9 & 
87.4 & 
83.5 & 
82.5 & 
86.9 & 
84.6 & 
82.4 & 
87.0 & 
84.6\tabularnewline

% 2 & 
% - & 
% - & 
% - & 
% - & 
% - & 
% - & 
% - & 
% - & 
% - & 
% - & 
% - & 
% -\tabularnewline

3 & 
- & 
- & 
- & 
\textbf{83.8} & 
\textbf{85.6} & 
\textbf{84.7} & 
\textbf{84.3} & 
85.6 & 
\textbf{84.9} & 
\textbf{83.3} & 
86.2 & 
\textbf{84.8}\tabularnewline

5 & 
\textbf{83.2} & 
88.8 & 
85.9 & 
- & 
- & 
- & 
- & 
- & 
- & 
81.5 & 
87.2 & 
84.3\tabularnewline

10 & 
\textbf{83.3} & 
\textbf{91.7} & 
\textbf{87.3} & 
78.1 & 
85.6 & 
81.7 & 
80.3 & 
\textbf{88.7} & 
84.3 & 
73.1 & 
84.9 & 
78.6\tabularnewline

15 & 
82.3 & 
90.2 & 
86.1 & 
- & 
- & 
- & 
- & 
- & 
- & 
- & 
- & 
-\tabularnewline
\hline
\end{tabular}
\end{table*}

\begin{table}[htbp]
\setlength{\tabcolsep}{4.0pt}
\footnotesize
\centering
\caption{Ablation of our proposed Iterative Text Decoding Network. "$\text{add}$" and "$\text{mul}$" indicate the cases where our refinement function is replaced by add and multiplication based refinement function respectively. "$\text{w/o reference}$" indicates that Fourier descriptor is not used as the reference location for the Multi-Scale Deformable Attention Module. "$\text{w/o refinement}$" indicates the case where the refinement module is removed.  }
\label{tab:refine}
\begin{tabular}{c|ccc|ccc}
\hline
\multirow{2}{*}{
Refinement module} & \multicolumn{3}{c|}{
CTW1500} & \multicolumn{3}{c}{
MSRA-TD500}\tabularnewline
%\cline{2-7}
 & 
R(\%) & 
P(\%) & 
F(\%) & 
R(\%) & 
P(\%) & 
F(\%)\tabularnewline
\hline

w/o refinement & 
78.9 & 
82.0 & 
80.4 & 
19.1 & 
9.2 & 
12.4\tabularnewline

w/o reference & 
76.4 & 
79.9 & 
78.1 & 
40.2 & 
43.3 & 
41.7\tabularnewline

add & 
77.7 & 
85.0 & 
81.2 & 
36.1 & 
47.3 & 
40.9\tabularnewline

mul & 
78.8 & 
85.4 & 
82.0 & 
50.3 & 
65.7 & 
57.0\tabularnewline
ours & 
\textbf{79.2} & 
\textbf{87.4} & 
\textbf{83.1} & 
\textbf{58.8} & 
\textbf{67.7} & 
\textbf{62.9} \tabularnewline
\hline
\end{tabular}
\end{table}

% ablation of loss function
\begin{table}[htbp]
\setlength{\tabcolsep}{3.1pt}
\small
\centering
\caption{Ablation of the loss functions. }
\label{tab:loss}
\begin{tabular}{ccc|ccc|ccc}
\hline
\multirow{2}{*}{
$\mathcal{L}_{SD}$} & \multirow{2}{*}{
$\mathcal{L}_{FD}$} & \multirow{2}{*}{
$\mathcal{L}_{bbox}$} & \multicolumn{3}{c|}{
CTW1500} & \multicolumn{3}{c}{
MSRA-TD500}\tabularnewline
%\cline{4-9}
 &  &  & 
R(\%) & 
P(\%) & 
F(\%) & 
R(\%) & 
P(\%) & 
F(\%)\tabularnewline
\hline

 & 
\checkmark & 
\checkmark & 
75.3 & 
84.4 & 
79.6 & 
56.7 & 
59.9 & 
58.3\tabularnewline

\checkmark & 
\checkmark & 
 & 
\textbf{82.6} & 
86.5 & 
84.5 & 
\textbf{82.0} & 
83.5 & 
82.7\tabularnewline

\checkmark & 
 & 
\checkmark & 
81.8 & 
86.2 & 
84.0 & 
81.6 & 
84.2 & 
82.9\tabularnewline

\checkmark & 
\checkmark & 
\checkmark & 
82.5 & 
\textbf{86.9} & 
\textbf{84.6} & 
79.2 & 
\textbf{88.3} & 
\textbf{83.5}\tabularnewline
\hline
\end{tabular}
\end{table}

\textbf{Iterative Text Decoding Network.}
As shown in Table~\ref{tab:refine}, replacing the refinement function by the refinement function based on addition or multiplication can result in the decrease of F-measure for the proposed FANet by 1.9\% and 1.1\% on CTW1500, and by 22.0\% and 5.9\% on MSRA-TD500 respectively, this shows that the proposed refinement module is more suitable for the normalized Fourier descriptor. Not taking the Fourier descriptor as the reference location for the Multi-Scale Deformable Attention module reduce the F-measure performance of the proposed FANet by 5.0\% and 21.2\% on CTW1500 and MSRA-TD500 respectively, which shows that aggregating features near the text region is conducive to the accurate regression of the network to the text instances.

% We compare our $\text{Refine}$ function with the original add based $\text{Refine}$ function. As shown in Table~\ref{tab:refine}, our refine method surpass original way over 0.8\% (82.0\% $vs.$ 81.2\%) and over 16.1\% (57.0\% $vs.$ 40.9\%) in F-measure on CTW1500 and MSRA-TD500 respectively, which indicates that it is more appropriate to use multiplication based refine instead of the addition based refine for the high frequencies.

\textbf{Loss function.}
As shown in Table~\ref{tab:loss}, removing $\mathcal{L}_{SD}$, $\mathcal{L}_{FD}$ or $\mathcal{L}_{bbox}$ in the matching cost calculation process and loss calculation process reduce the F-measure performance of the proposed FANet by 5.0\%, 0.6\% and 0.1\% on CTW1500 respectively, and by 25.2\%, 0.6\% and 0.8\% on MSRA-TD500 respectively. From this result, we can know that $\text{L1}$ loss for the resampled points of the text contour is more important than other two losses. Nevertheless, mixing the three losses submit the best performance, which indicates that the gains brought by the three losses are not completely overlapping. In order to achieve the 
best performance, we take the mixed loss of the three losses as the loss function of the proposed FANet.

\textbf{Dense matching strategy.}
As shown in Table~\ref{tab:convergence}, with only 30 epochs of training, the Dense Matching Strategy (DMS) can bring improvements of 3.3\%, 8.6\%, 3.5\% and 2.9\% based on F-measure on ICDAR2015, MSRA-TD500, CTW1500 and TotalText respectively. We attribute this to the fact that DMS can alleviate the overfitting problem caused by lack of positive samples. As the number of epochs increases, the gain brought by DMS becomes smaller, especially on CTW1500 and TotalText. We attribute this to the fact that as the number of training epochs increases, data enhancement can alleviate the performance degradation caused by overfitting 
and narrows the performance gap between HMS and DMS. So when we are faced with scenarios that enough training epochs are provided such as pre-training, using the original HMS or use DMS with small $H_m$ instead is a better choice. Match number $N_m$ in DMS is an important hyperparameter, 
compared with the case when $N_m=1$, setting $N_m$ to 10 will bring relative gain of $3.8\%$ $(87.3\% vs. 83.5\%)$ based on F-measure on MSRA-TD500, which shows the effectiveness of the DMS. Compared with the case when $N_m=10$, setting $N_m$ to 5 or 15 will reduce the performance of our network by $1.4\%$ $(87.3\% vs. 85.9\%)$ and $1.2\%$ $(87.3\% vs. 86.1\%)$ in F-measure respectively, which indicates that too large or too small $N_m$ will damage the performance of the network.
Compared with the case when $N_m=1$, setting $N_m$ to 3 will bring relative gain of $1.2\%$ $(84.7\% vs. 83.5\%)$, $0.3\%$ $(84.9\% vs. 84.6\%)$ and $0.2\%$ $(84.8\% vs. 84.6\%)$ in F-measure on ICDAR2015, CTW1500 and TotalText respectively. Same as MSRA-TD500, too large or too small $N_m$ will damage the performance of the network. As we can see in Table~\ref{tab:match}, the optimal choice of match number is 10 on MSRA-TD500 and 3 on ICDAR2015, CTW1500 and TotalText.
We can deduce that for a dataset with only a small number of text instances in each image like MSRA-TD500, the problem of overfitting is exacerbated by the lack of positive samples, so a larger $N_m$ can bring the network considerable performance gain, but for datasets with dense text instances in each image like ICDAR2015, CTW1500 and TotalText, a smaller $N_m$ is more appropriate.

\section{Conclusion}
This paper focus on the fast localization learning and accurate detection for scene text detection. We propose FANet, a  Fast convergence and Accurate scene text detection Network, which can achieve the SOTA performance on ICDAR2015, MSRA-TD500, CTW1500 and TotalText. More importantly, FANet can accurately detect scene text of arbitrary shapes with fewer training epochs.

\appendix

\section{Fourier Descriptor Normalization}
\begin{theorem}
If $\mathcal{Z}(t) = \mathcal{X}(t)+i\mathcal{Y}(t)$ is a univariate continuous periodic function about independent variable $t$, satisfying $t\in [0,1]$, the period is 1, $(\mathcal{X}(t),\mathcal{Y}(t))\in [0,1]^2$. Then we have the following conclusions:
\begin{align}
    c_k &= \int_0^1 \mathcal{Z}(t)e^{-2 \pi i k t}dt, k\in \sZ\\
    c_0 &= u_0+iv_0, u_0,v_0\in [0,1]\\
    c_l &= u_l+iv_l, u_l,v_l\in [-\frac{2}{\pi},\frac{2}{\pi}],l\in \sZ,l\neq 0
\end{align}
\end{theorem}
\begin{proof}
\begin{align}
    c_k &= \int_0^1 \mathcal{Z}(t)e^{-2 \pi i k t}dt\notag\\
    &= \int_0^1[\mathcal{X}(t)+i\mathcal{Y}(t)][cos(-2\pi k t)+isin(-2\pi k t)]dt\notag\\
    &= \int_0^1[\mathcal{X}(t)cos(2\pi k t)+\mathcal{Y}(t)sin(2\pi k t)]dt\notag\\
    & +i\int_0^1[-\mathcal{X}(t)sin(2\pi k t)+\mathcal{Y}(t)cos(2\pi k t)]dt\notag\\
    &= u_k+iv_k
\end{align}
When $k = 0$, there are the following conclusions:\\
% when $t\in [0,\frac{1}{4k}),(\mathcal{X}(t),\mathcal{Y}(t))=(1,1)$\\
% when $t\in [\frac{1}{4k},\frac{2}{4k}),(\mathcal{X}(t),\mathcal{Y}(t))=(0,1)$\\
% when $t\in [\frac{2}{4k},\frac{3}{4k}),(\mathcal{X}(t),\mathcal{Y}(t))=(0,0)$\\
% when $t\in [\frac{3}{4k},\frac{1}{k}),(\mathcal{X}(t),\mathcal{Y}(t))=(1,0)$\\
\begin{align}
    u_0 &= \int_0^1\mathcal{X}(t)dt\in [0,1]\\
    v_0 &= \int_0^1\mathcal{Y}(t)dt\in [0,1]
\end{align}

\begin{table*}[htbp]
\centering
\begin{equation}
\text{MSDeformAttn}(\vz_q, \hat{\vp}_q, \{\vx^l\}_{l=1}^{L}) = \sum_{m=1}^{M} \mW_m \big[\sum_{l=1}^{L} \sum_{s=1}^{S} \emA_{qmls} \cdot \mW'_m \vx^l(\phi_{l}(\hat{\vp}_q) + \Delta\vp_{qmls})\big]
\label{eq:deform_attn_fun}
\end{equation}
\end{table*}

\noindent When $k \neq 0$, there are the following conclusions:
\begin{align}
    |u_k| &= |\int_0^1[\mathcal{X}(t)cos(2\pi k t)+\mathcal{Y}(t)sin(2\pi k t)]dt|\notag\\
    &\leq |\int_0^1\mathcal{X}(t)cos(2\pi k t)dt|+|\int_0^1\mathcal{Y}(t)sin(2\pi k t)dt|\label{eq:7}
\end{align}
For term $\int_0^1\mathcal{X}(t)cos(2\pi k t)dt$, there are the following inequalities:
\begin{align}
    -\frac{1}{\pi}&= \int_0^1cos(2\pi k t)I[cos(2\pi k t)<0]dt\notag\\
    &\leq \int_0^1\mathcal{X}(t)cos(2\pi k t)I[cos(2\pi k t)<0]dt\notag\\
    &\leq \int_0^1\mathcal{X}(t)cos(2\pi k t)dt\notag\\
    &\leq \int_0^1\mathcal{X}(t)cos(2\pi k t)I[cos(2\pi k t)>0]dt\notag\\
    &\leq \int_0^1cos(2\pi k t)I[cos(2\pi k t)>0]dt\notag\\
    &=\frac{4k}{2} \big|\int_0^{\frac{1}{4k}}cos(2\pi k t)dt\big|\notag \\
    &=\frac{4k}{2}\big|\frac{1}{2\pi k}sin(2\pi k t)|_{t=\frac{1}{4k}}\big|=\frac{1}{\pi}
\end{align}
We can get $|\int_0^1\mathcal{X}(t)cos(2\pi k t)dt|\leq \frac{1}{ \pi}$, and we can easily prove it in the same way that $|\int_0^1\mathcal{Y}(t)sin(2\pi k t)dt|\leq \frac{1}{\pi}$. We substitute it into Eq \ref{eq:7}, then we can get:
\begin{align}
    |u_k| &\leq \frac{1}{\pi}+\frac{1}{\pi} = \frac{2}{\pi},k\in \sZ,k\neq 0
\end{align}
In the same way, we can easily prove that:
\begin{align}
    |v_k| &\leq \frac{2}{\pi},k\in \sZ,k\neq 0
\end{align}
\end{proof}

\section{Fourier descriptor as the reference location}
Let $\{\vx^l\}_{l=1}^{L}$ be the input multi-scale feature maps, where $\vx^l \in \sR^{C \times H_l \times W_l}$. Let $\hat{\vp}_q \in [0, 1]^2$ be the normalized coordinates of the reference point for each query element $q$, then the Multi-Scale Deformable Attention module~\cite{DBLP:conf/iclr/ZhuSLLWD21} is applied as Equation~\ref{eq:deform_attn_fun},

where $m$ indexes the attention head, $l$ indexes the input feature level, and $k$ indexes the sampling point.
$\Delta\vp_{qmls}$ and $\emA_{qmls}$ denote the sampling offset and attention weight of the $s$-th sampling point in the $l$-th feature level and the $m$-th attention head, respectively.
The scalar attention weight $\emA_{qmls}$ is normalized by $\sum_{l=1}^{L} \sum_{s=1}^{S} \emA_{qmls} = 1$.
Here, we use normalized coordinates $\hat{\vp}_q \in [0, 1]^2$ for the clarity of scale formulation, in which the normalized coordinates $(0, 0)$ and $(1, 1)$ indicate the top-left and the bottom-right image corners, respectively. Function $\phi_{l}(\hat{\vp}_q)$ in Equation~\ref{eq:deform_attn_fun} rescales the normalized coordinates $\hat{\vp}_q$ to the input feature map of the $l$-th level. The multi-scale deformable attention samples $LS$ points from multi-scale feature maps. In the proposed FANet, we consider the calculation process of for query $q$ in $d$-th transformer decoder layer. We first obtain the normalized Fourier descriptor predicted by the $(d-1)$-th decoder layer of query $q$, which we denote as $\vc=[u_{-K},v_{-K},\cdots,u_0,v_0,\cdots,u_K,v_K]\in \sR^{2(2K+1)}$, We then calculate the normalized coordinate points on the predict text contour in spatial domain by Inverse Discrete Fourier Transform (IDFT):
\begin{align}
    \vp_{qn} = \mathcal{F}^{-1}(\frac{n}{\emN},\vc),n=1,\cdots,\emN\label{eq:12}
\end{align}
where $\vp_{qn} = (x_{qn},y_{qn})$ is the $n$-th point on the text contour predicted by query $q$, $\mathcal{F}^{-1}$ is the Inverse Fourier Transform (IFT) function, $\emN$ is the sampling number on the text contour. We then calculate the bounding box of the text contour:
\begin{align}
    x_q &= E_i(x_{qn})\label{eq:13}\\
    y_q &= E_i(y_{qn})\label{eq:14}\\
    w_q &= max_i(x_{qn})-min_i(x_{qn})\label{eq:15}\\
    h_q &= max_i(y_{qn})-min_i(y_{qn})\label{eq:16}
\end{align}
Where $\{x_q,y_q,w_q,h_q\}$ is the bounding box of the text contour. We use $\phi_l((x_q,y_q))$ instead of $\phi_l(\hat{\vp}_q)$ as the new reference point for query $q$. The sampling offset $\Delta{\vp_{mlqh}}$ is also modulated by the box size, as $(\Delta{\vp_{qmlsx}}w_q,\Delta{\vp_{qmlsy}}h_q)$.

\section{Refinement function}
\begin{theorem}
If $y = f[x + f^{-1}(x_0)]$ is a continuous function about $x$, where $f$ is a differentiable function and $x_0\in \sR$ is a constant value, then:
\begin{align}
    \lim_{x\rightarrow 0}\frac{\partial y}{\partial x} = \frac{\partial f(z)}{\partial z}|_{z=f^{-1}(x_0)}
\end{align}
\end{theorem}
\begin{proof}
\begin{align}
    \lim_{x\rightarrow 0}\frac{\partial y}{\partial x} &= \lim_{x\rightarrow 0}\frac{\partial f(z)}{\partial z}|_{z=x+f^{-1}(x_0)}\notag\\
    &=\frac{\partial f(z)}{\partial z}|_{z=f^{-1}(x_0)}
\end{align}

\end{proof}

\begin{theorem}
If $y = f[e^{x}f^{-1}(x_0)]$ is a continuous function about $x$, where $f$ is a differentiable function and $x_0\in \sR$ is a constant value, then:
\begin{align}
    \lim_{x\rightarrow 0}\frac{\partial y}{\partial x} = \frac{\partial f(z)}{\partial z}|_{z=f^{-1}(x_0)}f^{-1}(x_0)
\end{align}
\end{theorem}
\begin{proof}
\begin{align}
    \lim_{x\rightarrow 0}\frac{\partial y}{\partial x} &= \lim_{x\rightarrow 0}\frac{\partial f(z)}{\partial z}|_{z=e^{x}f^{-1}(x_0)}f^{-1}(x_0)e^{x}\notag\\
    &=\frac{\partial f(z)}{\partial z}|_{z=f^{-1}(x_0)}f^{-1}(x_0)
\end{align}
\end{proof}

{\small
\bibliographystyle{ieee_fullname}
\bibliography{arxiv}
}
\end{document}